\documentclass{article}
\usepackage[final]{eurips_2025}

\usepackage{natbib}
\usepackage{graphicx}
\usepackage{bm}
\usepackage{amsmath}
\usepackage{soul}
\usepackage{amssymb}
\usepackage{forest}
\usepackage{mathtools}
\usepackage{listings}
\usepackage{subcaption}
\usepackage{placeins}
\usepackage{tikz}
\usepackage{xifthen}
\usepackage[utf8]{inputenc}
\usepackage{xcolor}
\usepackage{booktabs}
\usepackage{tabularx}
\usepackage{array}
\usepackage{caption}
\usepackage{threeparttable}
\usepackage[colorlinks=true, allcolors=blue]{hyperref}
\usepackage{amsmath, amsthm}
\usepackage[nameinlink]{cleveref}

\newtheoremstyle{exampstyle}
{1pt}{1pt}{\itshape}{}{\bfseries}{.}{.5em}{}

\newtheorem{theorem}{Theorem}
\newtheorem{lemma}[theorem]{Lemma}

\newcommand\lref[1]{\hyperref[#1]{Lemma~\ref*{#1}}}

\newcommand{\dif}[0]{\mathrm{d}}
\newcommand{\KL}[2]{\mathbb{D}[#1||#2]}

\usepackage{titlesec}
\titlespacing\section{0pt}{2pt plus 2pt minus 2pt}{0pt plus 2pt minus 2pt}
\titlespacing\subsection{0pt}{2pt plus 2pt minus 2pt}{0pt plus 2pt minus 2pt}

\usepackage[font=scriptsize,labelfont=bf]{caption}
\usepackage{enumitem}
\setlist{nosep}

\usepackage{algorithm}
\usepackage{algpseudocode}
\usepackage{adjustbox}
\algrenewcommand\algorithmicrequire{\textbf{Input:}}
\algrenewcommand\algorithmicensure{\textbf{Output:}}

\usepackage[]{todonotes}

\newcommand{\ent}[1]{\mathbb{H}[#1]}

\title{Active Inference is a Subtype of Variational Inference}
\author{%
Wouter W. L. Nuijten $^{1,2}$ \qquad Mykola Lukashchuk$^1$ \thanks{Equal contribution between authors.}\\
$^1$Department of Electrical Engineering, Eindhoven University of Technology, the Netherlands\\
$^2$Lazy Dynamics, Utrecht, the Netherlands\\
\texttt{w.w.l.nuijten@tue.nl}\\
}

\begin{document}

\maketitle
\begin{abstract}
    Automated decision-making under uncertainty requires balancing exploitation and exploration. Classical methods treat these separately using heuristics, while Active Inference unifies them through Expected Free Energy (EFE) minimization. However, EFE minimization is computationally expensive, limiting scalability. We build on recent theory recasting EFE minimization as variational inference, formally unifying it with Planning-as-Inference and showing the epistemic drive as a unique entropic contribution. Our main contribution is a novel message-passing scheme for this unified objective, enabling scalable Active Inference in factored-state MDPs and overcoming high-dimensional planning intractability.

\end{abstract}

\section{Introduction}

Automated decision-making under uncertainty is a central, long-standing challenge across control theory and artificial intelligence. When the system dynamics are well-known and deterministic, classical methods like Optimal Control \citep{bellman_theory_1954} and Model Predictive Control (MPC) \citep{cutler_dynamic_1979} provide principled frameworks for determining optimal actions. These approaches, focused primarily on minimizing a predefined cost function, have been elegantly unified under the Planning as Inference (PAI) paradigm \citep{toussaint_robot_2009, attias_planning_2003}, showing that control can be cast as a variational inference problem on a factor graph \citep{todorov_general_2008,levine_reinforcement_2018}.

However, the real-world challenge lies in environments where dynamics are stochastic or partially unknown. Popular methods that operate under unknown dynamics are rooted in  Reinforcement Learning \citep{sutton_reinforcement_1998}, which optimizes long-term utility through value function or policy estimation. These methods attempt to inject epistemic behavior by treating exploration as a distinct form of reward, as seen in Max-entropy Reinforcement Learning \citep{haarnoja_soft_2018,haarnoja_reinforcement_2017}. 

Active Inference proposes an alternative approach to planning under uncertainty \citep{friston_active_2015,parr_generalised_2019,da_costa_active_2020}. This framework provides a neurobiological explanation of intelligent behavior and posits that the optimal policy that balances exploitative and explorative behavior emerges when minimizing a quantity known as the Expected Free Energy (EFE) \citep{friston_action_2010,da_costa_active_2020}. However, the EFE is an objective that is defined over sequences of actions and does therefore not define a variational objective over beliefs that we can optimize.

Recently, an attempt has been made to redefine EFE-based planning as a standard Variational Free Energy \citep{de_vries_expected_2025} by adjusting the generative model by introducing epistemic priors. 

In this paper, we will take a closer look at the objective defined by \cite{de_vries_expected_2025} and frame it as a form of entropic inference, as defined by \cite{lazaro-gredilla_what_2024}. Afterwards, we will derive a message passing scheme that corresponds to the found formulation of Active Inference and which can be locally minimized on a Factor Graph. 

The main contributions of this paper are twofold: 
\begin{itemize}
    \item We formally reframe Active Inference's EFE minimization as a form of entropy-corrected variational inference, explicitly demonstrating that the epistemic drive corresponds to a unique entropic contribution within the variational objective.
    \item We derive a message-passing scheme for this unified objective. Crucially, this scheme introduces region-extended Bethe coordinates and an $r$-channel reparameterization coordinate, which together turn a degenerate conditional entropy into a local cross-entropy and render the overall objective computationally feasible for local optimization on a factor graph.
\end{itemize}

The rest of the paper is structured as follows: in \Cref{sec:adjusted_objective} we recover Active Inference as a form of variational inference similar to how planning is recovered in \cite{lazaro-gredilla_what_2024}. This illustrates that the epistemic drive introduced by Active Inference can be materialized as an entropic contribution to the variational objective. In \Cref{sec:message_passing} we will derive a message passing scheme to minimize the Active Inference objective, providing a method to implement scalable Active Inference.

For a definition of the terminology used in the rest of the paper, we refer the reader to \autoref{appx:background}.

\section{Active Inference as Entropy Corrected Inference} \label{sec:adjusted_objective}
In this section, we will rewrite the Variational Free Energy of an adjusted generative model as a form of entropy corrected inference, comparing it to other formulations of planning-as-inference and posing Active Inference as a separate method on the variational inference landscape.

 We will consider the following standard biased generative model
\begin{equation} \label{eq:generative_model}
    p(\bm{y}, \bm{x}, \theta, \bm{u}) \propto p(\theta) p(x_0) \prod_{t=1}^T p(y_t | x_t, \theta) p(x_t | x_{t-1}, u_t, \theta) p(u_t) \hat{p}_{x}(x_t) \hat{p}_{y}(y_t).
\end{equation}

Here, $\bm{y}$ are observations, $\bm{x}$ are latent states, $\theta$ are hidden parameters, and $\bm{u}$ is a sequence of control signals or actions. $\hat{p}_{x}(x_t)$ and $\hat{p}_{y}(y_t)$ represent goal priors on future states and observations, which can be proportional to a prespecified reward but do not necessarily need to be.

The variational objective defined by \citet{de_vries_expected_2025} manipulates the Variational Free Energy (VFE) of the model in \eqref{eq:generative_model} through the inclusion of epistemic priors. In this section, we demonstrate that the framework of \citet{de_vries_expected_2025} extends beyond Active Inference by reformulating their objective within the broader landscape of entropic inference introduced by \citet{lazaro-gredilla_what_2024}. 

In this entropic inference framework, all inference types minimize a common VFE \eqref{eq:vfe-definition} while differing only in their entropy corrections. Following this principle, \autoref{thm:bfe} shows that the VFE of the epistemic-prior-augmented generative model from \eqref{eq:adjusted_generative_model} can be equivalently expressed as the VFE of the original generative model plus specific entropy correction terms, thereby positioning Active Inference within the unified variational inference landscape of \autoref{tab:entropic-landscape}.

\begin{theorem} \label{thm:bfe} The variational objective presented in \cite{de_vries_expected_2025} (presented in \autoref{subsec:manipulated-vfe}) can be rearranged in the following way:
\begin{equation}
    F_{\tilde{p}}[q] = F_p[q] + \sum_{t=1}^{T} \ent{q(x_{t-1}, u_t)} -\ent{q(x_t, x_{t-1}, u_t)}  + \ent{q(y_t, x_t, \theta)}  - \ent{q(x_t, \theta)} 
\end{equation}
where $F_p[q]$ is the Variational Free Energy associated with the generative model.
\end{theorem}
\begin{proof}
Given in \autoref{appx:adjusted_objective_proof}.
\end{proof}

We are now in the position to compare Active Inference with other forms of entropic inference. Interestingly, by \lref{lemma:p_tilde_u}, 
an adjusted generative model with only $\tilde{p}(u_t)$ as entropic prior recovers a form of inference surprisingly similar to \cite{lazaro-gredilla_what_2024}. However, where the planning-as-inference objective defines a degenerate optimization procedure, this objective admits an optimization scheme. This point will be elaborated on in \Cref{sec:message_passing}. We will refer to this type of inference as Maximum Ambiguity (MaxAmb) planning, and with the inclusion of $\tilde{p}(x_t)$ and $\tilde{p}(y_t, x_t)$, recovers Active Inference. An overview of the different types of entropic inference is given in \autoref{tab:entropic-landscape}.

\begin{table}[ht]
    \centering
    \caption{Positioning Active Inference within the variational inference landscape. Following \citet{lazaro-gredilla_what_2024}, various inference methods can be expressed as energy minimization with different entropy corrections. Active Inference emerges as a natural extension that incorporates both planning and epistemic (ambiguity-reducing) terms. Note a slight difference from the exposition presented in \cite[Table 1]{lazaro-gredilla_what_2024}: there, the entropy correction is presented for the so-called energy term, but these two frameworks are trivially equivalent in the cases presented below. However, we find this table clearer when written as an entropic correction for VFE, because it becomes much easier to determine the degenerate schemes (this point will be elaborated in detail in \autoref{sec:message_passing}).}
    \begin{adjustbox}{width=1\textwidth}
    \begin{tabular}{l |  l}
    Type of inference &  Entropy correction (relative to VFE)\\
    \hline
    Marginal &  0 \\
    MAP & $\ent{q}$ \\
    Planning &  $\sum_{t=1}^T \ent{q(x_{t-1}, u_t)} - \ent{q(x_{t-1})}$ (\autoref{appx:gredilla_entropy_decomp})\\
    MaxAmb planning & $\sum_{t=1}^T \ent{q(x_{t-1}, u_t)} - \ent{q(x_t, x_{t-1}, u_t)}$\\
    Active Inference  & $\sum_{t=1}^{T} \ent{q(x_{t-1}, u_t)} -\ent{q(x_t, x_{t-1}, u_t)}  + \ent{q(y_t, x_t, \theta)}  - \ent{q(x_t, \theta)}$
    \end{tabular}
    \end{adjustbox}
    \label{tab:entropic-landscape}
\end{table}

Interestingly enough, the contributions from \lref{lemma:p_tilde_x} and \lref{lemma:p_tilde_x_y} contain the same terms with their signs flipped, where all contributions from $\tilde{p}(x_t)$ are canceled out. This warrants a revision of the epistemic priors $\tilde{p}(x_t)$ and $\tilde{p}(y_t, x_t)$. In \autoref{section:proof-replacment} we provide a proof that these two priors can be replaced by $\tilde{p}(x_t, \theta) = \exp \left( - \ent{q(y_t \mid x_t, \theta)} \right)$ without changing the inference objective. This re-arrangement is theoretically useful because it shows that parameters and states are actually not distinguished by entropic priors, and the only possible distinction could come from the generative model itself. With the landscape of entropic inference set up, we are in a position to derive a message-passing procedure corresponding to Active Inference.

\section{Deriving Message Passing} \label{sec:message_passing}

\begin{figure}
    \centering
    \includegraphics[width=0.8\linewidth]{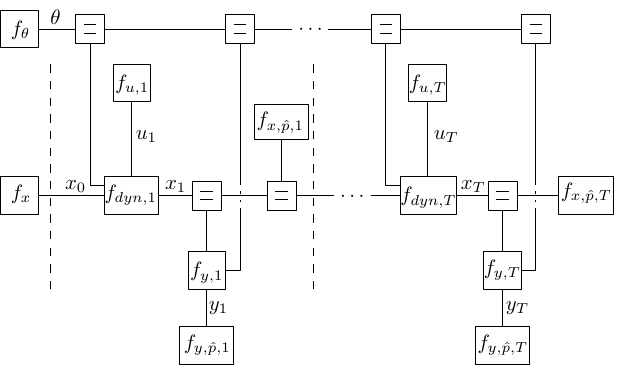}
    \caption{\textbf{Factor graph representation of the generative model \eqref{eq:generative_model}.} 
Nodes (boxes) represent factors from the generative model: $f_\theta$ is the prior on parameters, $f_{x_0}$ is the initial state prior, $f_{\mathrm{dyn},t}$ represents the dynamics $p(x_t|x_{t-1},\theta,u_t)$, $f_{y,t}$ represents observations $p(y_t|x_t,\theta)$, and $f_{u,t}$, $f_{x,\hat{p},t}$, $f_{y,\hat{p},t}$ represent action priors and goal priors respectively. 
Edges (lines) represent random variables: $\theta$ (parameters), $x_t$ (states), $y_t$ (observations), and $u_t$ (actions). 
In the Bethe approximation, each \emph{node} $a$ maintains a local belief $q_a(s_a)$ over its scope (the variables connected to it), while each \emph{edge} $i$ maintains a singleton belief $q_i(s_i)$. These local beliefs must satisfy consistency constraints \eqref{eq:local-consistency-constraint}. 
This factorization enables \emph{local} optimization scheme (message passing): rather than optimizing a single global distribution $q(y,x,\theta,u)$, we optimize a collection of local beliefs that communicate through messages.
}
    \label{fig:bethe-model}
\end{figure}

To obtain a \emph{local} objective amenable to message passing, we replace the global VFE in \autoref{thm:bfe} by its Bethe approximation presented in detail in \autoref{subsec:bethe-free-energy}. On tree-structured instances of \eqref{eq:generative_model} this replacement is \emph{exact} (for instance, a $\theta$-free model); otherwise, it is a standard variational approximation. But to define the Bethe objective, we need to identify our model with a factor graph (shown in \autoref{fig:bethe-model}). We start with the node set $\mathcal{V}$ \begin{subequations}\label{eq:our-factors}
\begin{align}
f_{\theta}(\theta)&=p(\theta), &
f_{x_0}(x_0)&=p(x_0),\\
f_{y,t}(y_t,x_t,\theta)&=p(y_t\mid x_t,\theta), &
f_{\mathrm{dyn},t}(x_t,x_{t-1},\theta,u_t)&=p(x_t\mid x_{t-1},\theta,u_t),\\
f_{u,t}(u_t)&=p(u_t), &
f_{x,\hat p,t}(x_t)&=\hat p_x(x_t), \quad
f_{y,\hat p,t}(y_t)=\hat p_y(y_t).
\end{align}
\end{subequations} With the natural set of edges $\mathcal{E}$ \begin{equation}
    \mathcal{E} := \{\theta,\ x_0\}\ \cup\ \bigcup_{t=1}^T \{x_t,\ y_t,\ u_t\}.
\end{equation} Each node \(a\in\mathcal{V}\) has scope \(s_a\subseteq\mathcal{E}\).

In Bethe terminology, each node $a \in \mathcal{V}$ has an associated \emph{local} belief $q_a(s_a)$ over its scope. Additionally, to impose consistency constraints between nodes, we introduce singleton beliefs $q_i(s_i)$ for each edge $i \in \mathcal{E}$. Together, these node beliefs $\{q_a(s_a)\}_{a \in \mathcal{V}}$ and edge beliefs $\{q_i(s_i)\}_{i \in \mathcal{E}}$ form what we call the \textbf{Bethe coordinates}:  a collection of normalized probability distributions that must satisfy the local consistency constraints:
\begin{equation}\label{eq:local-consistency-constraint}
    \int q_a(s_a)\dif s_{a\setminus i}=q_i(s_i)
\end{equation}
whenever $i \in s_a$.

With this notation, the Bethe Free Energy specialized to \eqref{eq:generative_model} takes the standard form
\[
F_f[q]\;=\;\sum_{a\in\mathcal{V}} \KL{q_a(s_a)}{f_a(s_a)}\;+\;\sum_{i\in\mathcal{E}}(d_i-1)\,\ent{q_i(s_i)},
\]
whose fully expanded expression and $d_{i}$ are given in \autoref{eq:bethe-our-model-updated} and \autoref{eq:degrees-our-model-updated} respectively.

But to make the objective from \autoref{thm:bfe} local, we must express all its terms using local marginals; otherwise, it is a global objective. Intuitively, this means we need to find a node in our factor graph to which we can attach each new term. For instance, $q(y_t, x_t, \theta)$ can be attached to the node $f_{y,t}$ and be identified with $q_{y,t}$ Bethe coordinate.

However, the Bethe coordinates alone are \emph{insufficient} for the adjusted objective in \autoref{thm:bfe}, because the entropic correction contains
\[
 -\ \ent{q(x_t,\theta)}\ +\ \ent{q(x_{t-1},u_t)}\ -\ \ent{q(x_t,x_{t-1},u_t)}.
\]
These three entropy terms involve marginal distributions that are not Bethe coordinates: $q(x_t,\theta)$, $q(x_{t-1},u_t)$, and $q(x_t,x_{t-1},u_t)$. None of these distributions correspond to the scope of any node in our factor graph, nor are they singleton beliefs over edges. Note, the same exact reasoning applies to Planning and MaxAbm Planning from \autoref{tab:entropic-landscape}.

To keep the objective local, we therefore introduce three auxiliary \emph{region beliefs}:
\begin{subequations}\label{eq:aux-beliefs}
\begin{align}
q_{\mathrm{sep},t}(x_t,\theta)
&:= \int q_{y,t}(y_t,x_t,\theta)\,\mathrm{d}y_t
 =  \int q_{\mathrm{dyn},t}(x_t,x_{t-1},\theta,u_t)\,\mathrm{d}x_{t-1}\mathrm{d}u_t,\\
q_{\mathrm{trip},t}(x_t,x_{t-1},u_t)
&:= \int q_{\mathrm{dyn},t}(x_t,x_{t-1},\theta,u_t)\,\mathrm{d}\theta, \\
q_{\mathrm{pair},t}(x_{t-1},u_t)
&:= \int q_{\mathrm{trip},t}(x_t,x_{t-1},u_t)\,\mathrm{d}x_t
\end{align}
\end{subequations}
with the natural projections (e.g., \(\int q_{\mathrm{trip},t}\,\mathrm{d}x_t=q_{\mathrm{pair},t}\), \(\int q_{\mathrm{sep},t}\,\mathrm{d}x_t=q_{\theta}\)) enforcing consistency with existing beliefs. We will refer to this coordinate system as the \emph{region-extended Bethe coordinates}. In the region-extended Bethe coordinates, the objective from \autoref{thm:bfe} can be expressed as follows\[
F_{f}[q]
\;+\;\sum_{t=1}^T\Big(
\ent{q_{y,t}}
\ -\ \ent{q_{\mathrm{sep},t}}
\ +\ \ent{q_{\mathrm{pair},t}}
\ -\ \ent{q_{\mathrm{trip},t}}
\Big).
\]

However, after adding the new region marginals, we obtain the local objective that is degenerate because of the term $\ent{q_{y,t}}$; we prove this supporting result \autoref{thm:obs-degeneracy-aug} in \autoref{subsec:deg-of-marginal-system}. To have both locality and full-identifiability, we augment the coordinates with a single \emph{channel} variable \(r_{y\mid x\theta,t}(y_t\mid x_t,\theta)\) with the natural normalization constraint \begin{equation}\label{eq:normalization-constraint}
    \int r_{y\mid x\theta,t}(y_{t}\mid x_{t}, \theta) \dif y_{t} = 1
\end{equation} 
 and rewrite the global conditional entropy as a local cross-entropy,
\[
H\big[q(y_t\mid x_t,\theta)\big]
\;=\;
\min_{r_{y\mid x\theta,t}}
\ \mathbb{E}_{q_{y,t}(y_t,x_t,\theta)}
\big[-\log r_{y\mid x\theta,t}(y_t\mid x_t,\theta)\big].
\]
Under this reparameterization, $q_{\mathrm{sep},t}(x_t,\theta)$ is no longer needed as a free coordinate in the objective: the global term \(H[q(y_t,x_t,\theta)]-H[q(x_t,\theta)]\) collapses into the conditional form and depends only on \((q_{y,t},r_{y\mid x\theta,t})\) locally. The following theorem shows the stationary conditions in the $r$-adjusted coordinate system. The proof of \autoref{thm:update-scheme} is provided in \autoref{appd:proof-scheme-theorem}.



\begin{theorem}[The stationary scheme for Active Inference]\label{thm:update-scheme}
Consider the Bethe objective \eqref{eq:bethe-our-model-updated} for the model \eqref{eq:generative_model} augmented by the Active Inference correction of \autoref{thm:bfe}, and adopt the adjusted coordinate system
\begin{align*}
\Big\{\,&q_{y,t}(y_t,x_t,\theta),\;q_{\mathrm{dyn},t}(x_t,x_{t-1},\theta,u_t),\;
q_{\mathrm{sep},t}(x_t,\theta), \\
&q_{\mathrm{trip},t}(x_t,x_{t-1},u_t),\;q_{\mathrm{pair},t}(x_{t-1},u_t),\;
r_{y\mid x\theta,t}(y_t\mid x_t,\theta)\,\Big\}_{t=1}^T
\end{align*}
with the projection constraints \autoref{eq:aux-beliefs}
and the row–normalization \autoref{eq:normalization-constraint}.
Then any stationary point satisfies, for each \(t=1,\dots,T\), the following local equations (all equalities are up to normalizers):

\begin{align}
&q_{y,t}(y_t,x_t,\theta)
\propto
p(y_t\mid x_t,\theta)\;
r_{y\mid x\theta,t}(y_t\mid x_t,\theta)\;
\hat{p}_{y}(y_{t})\;
\exp\{-\Lambda_{x\theta}(x_t,\theta)\},
\label{eq:AI-scheme-qy}\\[2pt]
&r_{y\mid x\theta,t}(y_t\mid x_t,\theta)
\propto \frac{q_{y,t}(y_t,x_t,\theta)}{q_{\mathrm{sep},t}(x_t,\theta)}
\label{eq:AI-scheme-r}\\[2pt]
&\exp\{-\Lambda_{x\theta}(x_t,\theta)\}
\propto
\frac{\displaystyle \int p(y_t\mid x_t,\theta)\,r_{y\mid x\theta,t}(y_t\mid x_t,\theta)\,\hat{p}_{y}(y_{t})\,{\rm d}y_t}
{\,q_{\mathrm{sep},t}(x_t,\theta)\,}.
\label{eq:AI-scheme-Lxth}
\end{align}

\begin{align}
&q_{\mathrm{dyn},t}(x_t,x_{t-1},\theta,u_t)
\propto
p(x_t\mid x_{t-1},\theta,u_t)\;
\exp\{-\Lambda_{x\theta}(x_t,\theta)\}\;
\exp\{-\Lambda_{\mathrm{trip}}(x_t,x_{t-1},u_t)\},
\label{eq:AI-scheme-qdyn}\\[2pt]
&\exp\{-\Lambda_{\mathrm{trip}}(x_t,x_{t-1},u_t)\}
\propto
\frac{q_{\mathrm{pair},t}(x_{t-1},u_t)}{q_{\mathrm{trip},t}(x_t,x_{t-1},u_t)}.
\label{eq:AI-scheme-Ltrip}
\end{align}
The region beliefs are tied by the projections
\(\int q_{\mathrm{dyn},t}\,{\rm d}\theta=q_{\mathrm{trip},t}\) and
\(\int q_{\mathrm{trip},t}\,{\rm d}x_t=q_{\mathrm{pair},t}\).

\paragraph{All remaining coordinates.}
Singletons \(q_{x_t}, q_{y_t}, q_{u_t}, q_{\theta}, q_{x_0}\) and unary factor beliefs (including \(\hat p_x,\hat p_y\)) satisfy the classical Bethe equations, i.e.\ the standard belief–propagation fixed–point conditions on the generative model; equivalently, their multipliers/messages are exactly those of BP on \eqref{eq:generative_model} (with degrees \eqref{eq:degrees-our-model-updated}) and are not modified by the entropic correction.

\end{theorem}

\section{Discussion}\label{sec:discussion}

While our theoretical framework provides principled planning with epistemic objectives, its computational implementation faces significant challenges that warrant careful analysis.

To implement our scheme, we must address the nontrivial factor graph structure shown in \autoref{fig:factor_graph}.
This is a Forney-style factor graph \citep{forney_codes_2001} representing a single time slice, where variables are represented on edges and factors are represented as nodes.
Message passing algorithms are generally implemented on factor graphs to leverage their locality \citep{bagaev_reactive_2023}.
However, the factor graph corresponding to the scheme introduced in \Cref{sec:message_passing} is nontrivial.
Unlike standard factor graphs derived directly from generative models, our approach requires region-based representations \citep{yedidia_constructing_2005} with edges representing multiple variables.
In \autoref{fig:factor_graph}, we see the edges representing $(x_{t-1}, \theta)$ and $(x_t, \theta)$.
Furthermore, there are additional nodes (shown in dotted lines in \autoref{fig:factor_graph}) that are necessary to compute the new coordinates in our optimization procedure.
The functional form of these nodes is currently unknown, but the messages that are sent are defined by the schedule derived in \autoref{thm:update-scheme}.
This means we cannot interpret what these nodes do concretely, highlighting a gap between our theoretical framework and its practical implementation.

Previous attempts have been made to implement a minimization of the objective presented in \cite{de_vries_expected_2025}.
These attempts have also implemented a message-passing procedure on a factor graph \citep{nuijten_message_2025}.
The scheme previously derived manually recomputes the epistemic priors for each iteration of the variational inference procedure.
Still, the scheme can be viewed as a linearization of the true message-passing scheme, which is derived in this work.

Beyond structural challenges, the computational complexity of our approach is quadratic in the state space size.
Previous derivations of Planning-as-Inference express the computational cost of the procedure in terms of the number of variables \citep{lazaro-gredilla_what_2024}.
We argue that this is a misleading way to express cost, as the cost of computing joint marginal distributions greatly depends on the size and dimensionality of the state space.
The scheme introduced in \Cref{sec:message_passing} warrants the computation of $q_{y,t}(y_t, x_t, \theta)$, $q_{pair,t}(x_{t-1}, u_t)$ and $q_{trip,t}(x_t, x_{t-1}, u_t)$.
In discrete state and action spaces, the computational complexity of computing these quantities is polynomial in the state and action spaces.
The computation of $q_{trip,t}(x_t, x_{t-1}, u_t)$ is the most expensive, since it takes $\mathcal{O}(|X|^2 \cdot |U| \cdot D_\Theta)$ time, where $D_\Theta$ is the dimensionality of the parameter space.
Note that, if we would not localize the inference procedure and not introduce our message passing scheme, the computational complexity would be exponential in the planning horizon $T$. 
Interestingly enough, the scheme for Planning as Inference also requires this time complexity, as the same $q_{trip,t}(x_t, x_{t-1}, u_t)$ is computed.
The epistemic drive, however, elicits an additional complexity cost. Computing $q_{y,t}(y_t, x_t, \theta)$ takes $\mathcal{O}(|Y| \cdot |X| \cdot D_\Theta)$ time.
This quadratic dependence on the state space is limiting for interesting problems where the state space quickly grows with the size of the system, such as Minigrid environments \citep{chevalier-boisvert_minigrid_2023}.

These complexity limitations suggest that for this approach to truly scale, hierarchical state-space partitioning becomes essential.
The scheme derived in this paper warrants a partition of the state-space, hinting to a hierarchical generative model \citep{palacios_markov_2020,beck_dynamic_2025}. Such hierarchical partitioning would allow us to avoid the quadratic complexity within state-space partitions, dramatically reducing computational costs while maintaining the principled epistemic objectives of our framework.

\begin{figure}[bt]
    \centering
    \includegraphics[width=0.70\textwidth]{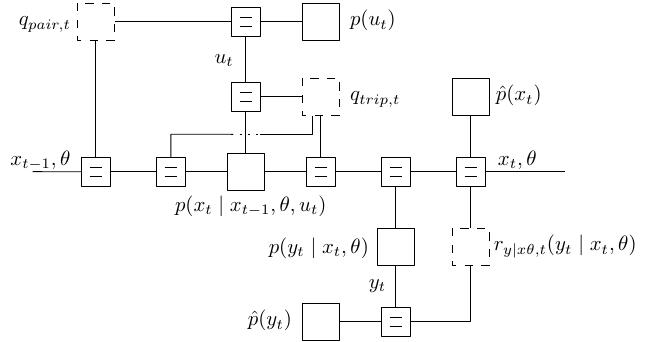}
    \caption{Time-slice factor graph corresponding to the scheme introduced in \autoref{sec:message_passing}. To form a full generative model to run inference, we chain $T$ of these slices to form a terminated factor graph.}
    \label{fig:factor_graph}
\end{figure}

\section*{Acknowledgements}
    This publication is part of the project ROBUST: Trustworthy AI-based Systems for Sustainable Growth with project number KICH3.LTP.20.006, which is (partly) financed by the Dutch Research Council (NWO), GN Hearing, and the Dutch Ministry of Economic Affairs and Climate Policy (EZK) under the program LTP KIC 2020-2023.

\bibliography{references}

\appendix
\numberwithin{equation}{section}

\section{Background} \label{appx:background}
\subsection{Variational Inference and the Posterior Factorization}
In standard Variational Inference, we minimize the Variational Free Energy (VFE) between a variational posterior $q$ and a generative model $p$:
\begin{equation} \label{eq:vfe-definition}
    F_p[q] = \int q(\bm{y}, \bm{x}, \theta, \bm{u}) \log \frac{q(\bm{y}, \bm{x}, \theta, \bm{u})}{p(\bm{y}, \bm{x}, \theta, \bm{u})} \dif \bm{y} \dif \bm{x} \dif \theta \dif \bm{u}.
\end{equation}
We are considering the factorized generative model defined in \eqref{eq:generative_model}. Since \eqref{eq:vfe-definition} defines a functional objective that we can minimize, we should also specify a family $\mathcal{Q}$ over which we are optimizing the VFE. We will choose the elements $q$ of this family such that they decompose as follows 
\begin{equation}\label{eq:posterior-factorization}
    q(\bm{y}, \bm{x}, \theta, \bm{u}) = q(x_0, \theta)\prod_{t=1}^{T} q(y_t \mid x_t, \theta) q(x_t \mid x_{t-1}, u_t, \theta) q(u_t \mid x_{t-1}, \theta).
\end{equation}
This factorization of the posterior distribution is required for the definition of the augmented generative model in \autoref{subsec:manipulated-vfe} \citep{nuijten_message_2025}. 
\subsection{The Epistemic-Prior-Augmented Generative Model} \label{subsec:manipulated-vfe}

The key insight of \citet{de_vries_expected_2025} is that Active Inference can be recovered through an \emph{adjusted} or \emph{augmented} generative model that includes additional factors called \emph{epistemic priors}. From the generative model, an augmented model is constructed
\begin{equation} 
    \tilde{p}(\bm{y}, \bm{x}, \theta, \bm{u}) = p(\bm{y}, \bm{x}, \theta, \bm{u}) \cdot \prod_{t=1}^T \tilde{p}(x_t) \tilde{p}(u_t) \tilde{p}(y_t, x_t).
    \label{eq:adjusted_generative_model}
\end{equation}
Here, the additional $\tilde{p}$ terms are epistemic priors, which are functions of the variational distribution $q$ itself, creating a self-consistent optimization problem. The epistemic priors take the following forms:
\begin{subequations}
\begin{align}
    \tilde{p}(u_t) &= \exp\left(\ent{q(x_t, x_{t-1} | u_t)} - \ent{q(x_{t-1} | u_t)}\right) && \text{(action prior)}; \\
    \tilde{p}(x_t) &= \exp\left(-\ent{q(y_t | x_t)}\right) && \text{(state prior)}; \\
    \tilde{p}(y_t, x_t) &= \exp\left(D_{KL}[q(\theta | y_t, x_t) \| q(\theta | x_t)]\right) && \text{(observation prior)},
\end{align}
\end{subequations} where the entropy of a distribution $q$ over variables $z_{1}, \dots, z_{n}$ definied as follows \begin{equation}
    \ent{q(z_1, \dots, z_{n})} = - \int q(z_{1}, \dots, z_{n}) \log q(z_{1}, \dots, z_{n}) \dif z_{1} \dots \dif z_{n},
\end{equation}
and the conditional entropy has the following functional form \begin{equation}
        \ent{q(z_1, \dots, z_{n} \mid \omega)} = - \int q(z_{1}, \dots, z_{n} \mid \omega) \log q(z_{1}, \dots, z_{n} \mid \omega) \dif z_{1} \dots \dif z_{n}.
\end{equation}
The epistemic priors admit a practical interpretation: the action prior encourages actions that resolve ambiguity, the state prior favors informative states, and the observation prior encourages observations that are informative about the parameters. 

The \emph{adjusted VFE} is then defined in the following way:
\begin{equation}
    F_{\tilde{p}}[q] = \int q(\bm{y}, \bm{x}, \theta, \bm{u}) \log \frac{q(\bm{y}, \bm{x}, \theta, \bm{u})}{\tilde{p}(\bm{y}, \bm{x}, \theta, \bm{u})} \dif \bm{y} \dif \bm{x} \dif \theta \dif \bm{u}.
\end{equation}
Crucially, \citet{de_vries_expected_2025} show that minimizing this adjusted VFE is equivalent to minimizing a bound on the Expected Free Energy (EFE) objective from Active Inference \citep{friston_active_2015}. 

\subsection{Bethe Free Energy}\label{subsec:bethe-free-energy}

Let $(\mathcal{V},\mathcal{E})$ be the (Forney-style) factor graph of a positive function
$f(s)=\prod_{a\in\mathcal{V}} f_a(s_a)$, where $s_a$ collects the variables incident on factor $a$, and each edge $i\in\mathcal{E}$ carries a variable $s_i$. The \emph{Bethe variational family} consists of \emph{factor beliefs} $\{q_a(s_a)\}_{a\in\mathcal{V}}$ and \emph{edge beliefs} $\{q_i(s_i)\}_{i\in\mathcal{E}}$ constrained to the \emph{marginal manifold}:
\begin{subequations}\label{eq:bethe-manifold}
\begin{align}
&\int q_a(s_a)\,\mathrm{d}s_a = 1 \quad &&\forall a\in\mathcal{V},\\
&\int q_i(s_i)\,\mathrm{d}s_i = 1 \quad &&\forall i\in\mathcal{E},\\
&\int q_a(s_a)\,\mathrm{d}s_{a\setminus i} = q_i(s_i) \quad &&\forall a\in\mathcal{V},\ \forall i\in a,
\end{align}
\end{subequations}
where $s_{a\setminus i}$ denotes the variables in $s_a$ except $s_i$. Let $d_i$ be the number of factors incident on edge $i$ (the degree of variable $s_i$). The \emph{Bethe Free Energy} (BFE) \citep{yedidia_constructing_2005} is
\begin{equation}\label{eq:classical-bethe}
F_f[q] \;=\; \sum_{a\in\mathcal{V}} \KL{q_{a}(s_a)}{f_{a}(s_{a})} \;+\; \sum_{i\in\mathcal{E}} (d_i-1)\,\ent{q_i(s_i)},
\end{equation} where \begin{equation}
    \KL{q_{a}(s_a)}{f_{a}(s_{a})} = \int q_{a}(s_{a}) \log \frac{q_{a}(s_{a})}{f_{a}(s_a)} \dif s_{a}.
\end{equation} The \emph{Bethe approximation} takes $\log Z_f \approx -\min_{q\in\mathcal{M}} F_f[q]$, where $\mathcal{M}$ is the manifold \eqref{eq:bethe-manifold}. The approximation is \emph{exact} when the factor graph is a tree; on loopy graphs, stationary points of \eqref{eq:classical-bethe} coincide with fixed points of (loopy) belief propagation \citep{yedidia_constructing_2005}. We refer the reader to \cite{senoz_message_2022} for a modern variational calculus exposition on the message passing algorithms derivation from \eqref{eq:classical-bethe}.
\section{Proof of \autoref{thm:bfe}} \label{appx:adjusted_objective_proof}
To prove the theorem, we will require three lemmas:
\begin{lemma} \label{lemma:p_tilde_x}
Under the assumption that our posterior distribution factorizes as in \autoref{eq:posterior-factorization}, and $\tilde{p}(x_t) = \exp \left( -\ent{q(y_t \mid x_t)} \right)$:
\begin{equation}
        -\int q(x_{t}) \log \tilde{p}(x_t) \dif x_t = \ent{q(y_t, x_t)} - \ent{q(x_t)}
\end{equation}

\end{lemma}
\begin{proof}
    \begin{subequations}
    \begin{align}
        -\int & q(x_{t}) \log \tilde{p}(x_t)  \dif x_t              \\
              & = -\int q(x_{t}) \int q(y_t \mid x_t) \log q(y_t \mid x_t) \dif y_t \dif x_t             \\
              & = -\int q(y_t \mid x_t) q(x_{t}) \log q(y_t \mid x_t) \dif y_t \dif x_t                  \\
              & = -\int q(y_t, x_{t}) \log \frac{q(y_t, x_t)}{q(x_t)} \dif y_t \dif x_t         \\
              & = -\int q(y_t, x_{t}) \log \frac{q(y_t, x_t)}{q(x_t)} \dif y_t \dif x_t  = \ent{q(y_t, x_t)} - \ent{q(x_t)}.
    \end{align}
\end{subequations}
\end{proof}

\begin{lemma}\label{lemma:p_tilde_u}
    Under the assumption that our posterior distribution factorizes as in \autoref{eq:posterior-factorization}, and $\tilde{p}(u_t) = \exp \left(\ent{q(x_t, x_{t-1}\mid u_t)}  - \ent{q(x_{t-1} \mid u_t)}\right)$ the following identity holds:
    \begin{equation}
        - \int q(u_t) \log \tilde{p}(u_t) \dif u_t = \ent{q(x_{t-1}, u_t)} -\ent{q(x_t, x_{t-1}, u_t)}
    \end{equation}
\end{lemma}
\begin{proof}
    \begin{subequations}
        \begin{align}
            - \int & q(u_t) \log \tilde{p}(u_t) \dif u_t \\
            & = \int q(u_t) \bigg ( \int q(x_t, x_{t-1} \mid u_t) \log q(x_t, x_{t-1} \mid u_t) \dif x_t \dif x_{t-1} \notag \\
            & \qquad - \int q(x_{t-1} \mid u_t) \log q(x_{t-1} \mid u_t) \dif x_{t-1} \bigg ) \dif u_t \\
            & = \int q(u_t) \bigg ( \int \frac{q(x_t, x_{t-1}, u_t)}{q(u_t)} \log \frac{q(x_t, x_{t-1}, u_t)}{q(u_t)} \dif x_t \dif x_{t-1} \notag \\
            & \qquad - \int \frac{q(x_{t-1}, u_t)}{q(u_t)} \log \frac{q(x_{t-1}, u_t)}{q(u_t)} \dif x_{t-1} \bigg ) \dif u_t \\
            & = \int q(x_t, x_{t-1}, u_t) \log \frac{q(x_t, x_{t-1}, u_t)}{q(u_t)} \dif x_t \dif x_{t-1} \dif u_t - \int q(x_{t-1}, u_t) \log \frac{q(x_{t-1}, u_t)}{q(u_t)} \dif x_{t-1} \dif u_t \\
            & = \underbrace{\int q(x_t, x_{t-1}, u_t) \log q(x_t, x_{t-1}, u_t) \dif x_t \dif x_{t-1} \dif u_t}_{-\ent{q(x_t, x_{t-1}, u_t)}} + \underbrace{\int q(x_t, x_{t-1}, u_t) \log \frac{1}{q(u_t)}\dif x_t \dif x_{t-1} \dif u_t}_{\ent{q(u_t)}} \notag \\
            & \qquad - \underbrace{\int q(x_{t-1}, u_t) \log q(x_{t-1}, u_t) \dif x_{t-1}' \dif u_t}_{-\ent{q(x_{t-1}, u_t}} - \underbrace{\int q(x_{t-1}, u_t) \log \frac{1}{q(u_t)}\dif x_t \dif x_{t-1} \dif u_t}_{\ent{q(u_t)}} \\
            & = \ent{q(x_{t-1}, u_t)} -\ent{q(x_t, x_{t-1}, u_t)}
        \end{align}
    \end{subequations}
\end{proof}

\begin{lemma} \label{lemma:p_tilde_x_y}
    Under the assumption that our posterior distribution factorizes as in \autoref{eq:posterior-factorization}, and $\tilde{p}(y_t, x_t) = \exp D_{KL}[q(\theta \mid y_t, x_t) || q(\theta \mid x_t)]$:
    \begin{equation}
        \begin{split}
                - \int  q(y_t,x_t,) \log \tilde{p}(y_t, x_t)  \dif y_t \dif x_t  
                 = \ent{q(y_t, x_t, \theta)} + \ent{q(x_t)} - \ent{q(y_t, x_t)} - \ent{q(x_t, \theta)}
        \end{split}
    \end{equation}
\end{lemma}
\begin{proof}
    \begin{subequations}
        \begin{align}
           - \int & q(y_t,x_t) \log \tilde{p}(y_t, x_t)  \dif y_t \dif x_t  \\
           &= - \int q(y_t,x_t)\left ( \int q(\theta \mid y_t, x_t) \log \frac{q(\theta \mid y_t, x_t)}{q(\theta \mid  x_t)} \dif \theta \right ) \dif y_t \dif x_t  \\
           &= - \int q(y_t, x_t) \left( \int q(\theta \mid y_t, x_t) \log q(y_t, x_t, \theta) - \log q(y_t, x_t) \dif \theta \right) \dif y_t \dif x_t \notag \\
           &\qquad + \int q(y_t, x_t) \left ( \int q(\theta \mid y_t, x_t) \log q(\theta, x_t) - \log q(x_t) \dif \theta \right ) \dif y_t \dif x_t \\
           & = \underbrace{-\int q(y_t, x_t, \theta) \log q(y_t, x_t, \theta) \dif y_t \dif x_t \dif \theta}_{\ent{q(y_t, x_t, \theta)}} + \underbrace{\int q(y_t, x_t, \theta) \log q(y_t, x_t) \dif y_t \dif x_t \dif \theta}_{-\ent{q(y_t, x_t)}} \notag \\
           &\qquad + \underbrace{\int q(y_t, x_t, \theta) \log q(x_t, \theta) \dif y_t \dif x_t \dif \theta}_{-\ent{q( x_t, \theta)}}  \underbrace{-\int q(y_t, x_t, \theta) \log q(x_t) \dif y_t \dif x_t \dif \theta}_{\ent{q(x_t)}} \\
           &= \ent{q(y_t, x_t, \theta)} - \ent{q(y_t, x_t)} -\ent{q( x_t, \theta)} + \ent{q(x_t)}
        \end{align}
    \end{subequations}
\end{proof}

Now that we have our lemmas in place, we can continue with the proof of \autoref{thm:bfe}.
\begin{proof}
    \begin{subequations}
        \begin{align}
            \tilde{F}_{\tilde{p}}[q] &= \int q(\bm{y}, \bm{x}, \theta, \bm{u}) \log \frac{q(\bm{y}, \bm{x}, \theta, \bm{u})}{\tilde{p}(\bm{y}, \bm{x}, \theta, \bm{u})} \dif \bm{y} \dif \bm{x} \dif \theta \dif \bm{u} \\
            &= \int q(\bm{y}, \bm{x}, \theta, \bm{u}) \log \frac{q(\bm{y}, \bm{x}, \theta, \bm{u})}{p(\bm{y}, \bm{x}, \theta, \bm{u}) \prod_{t=1}^T \tilde{p}(x_t)\tilde{p}(u_t) \tilde{p}(y_t, x_t)} \dif \bm{y} \dif \bm{x} \dif \theta \dif \bm{u} \\
            &= \underbrace{\int q(\bm{y}, \bm{x}, \theta, \bm{u}) \log \frac{q(\bm{y}, \bm{x}, \theta, \bm{u})}{p(\bm{y}, \bm{x}, \theta, \bm{u}) } \dif \bm{y} \dif \bm{x} \dif \theta \dif \bm{u}}_{F_p[q]} \notag \\
            &\qquad + \int q(\bm{y}, \bm{x}, \theta, \bm{u}) \log \frac{1}{\prod_{t=1}^T \tilde{p}(x_t)\tilde{p}(u_t) \tilde{p}(y_t, x_t)} \dif \bm{y} \dif \bm{x} \dif \theta \dif \bm{u} \\
            &= F_p[q] - \sum_{t=1}^T \int q(\bm{y}, \bm{x}, \theta, \bm{u}) \left( \log \tilde{p}(x_t) + \log \tilde{p}(u_t) + \log \tilde{p}(y_t, x_t) \right) \\
            &= F_p[q] + \sum_{t=1}^T \notag \\
            &\qquad - \int q(x_t) \log \tilde{p}(x_t)  \dif x_{t} - \int q( u_t) \log \tilde{p}(u_t) \dif u_t - \int q(y_t, x_t) \log \tilde{p}(y_t, x_t) \dif y_t \dif x_t.
        \end{align}
    \end{subequations}
    Here, we can recognize the identities from \lref{lemma:p_tilde_x}, \lref{lemma:p_tilde_u} and \lref{lemma:p_tilde_x_y}.
    \begin{subequations}
        \begin{align}
            F_p[q] &+ \sum_{t=1}^T - \int q(x_t) \log \tilde{p}(x_t)  \dif x_{t} 
             - \int q( u_t) \log \tilde{p}(u_t) \dif u_t - \int q(y_t, x_t) \log \tilde{p}(y_t, x_t) \dif y_t \dif x_t\\
            &= F_p[q] + \sum_{t=1}^{T} \ent{q(y_t, x_t)} - \ent{q(x_t)} + \ent{q(x_{t-1}, u_t)} -\ent{q(x_t, x_{t-1}, u_t)} \notag \\
            &\qquad + \ent{q(y_t, x_t, \theta)} + \ent{q(x_t)} - \ent{q(y_t, x_t)} - \ent{q(x_t, \theta)} \\
            &= F_p[q] + \sum_{t=1}^{T} \ent{q(x_{t-1}, u_t)} -\ent{q(x_t, x_{t-1}, u_t)}  + \ent{q(y_t, x_t, \theta)}  - \ent{q(x_t, \theta)} 
        \end{align}
    \end{subequations}
\end{proof}

\section{Additional Entropy Terms in Planning Objective} \label{appx:gredilla_entropy_decomp}
The adjusted inference objective in \cite{lazaro-gredilla_what_2024} is phrased as a maximization problem including conditional entropies. In this appendix, we will rephrase the planning-as-inference objective into the vocabulary used in this paper. \cite{lazaro-gredilla_what_2024} phrases their optimization problem as a maximization of a variational bound, whereas we pose the problem as a minimization of the Variational Free Energy. This means the entropy terms in \cite{lazaro-gredilla_what_2024} have their sign flipped, as this entropy term is subtracted from the objective. Note that we also have a slightly different indexing for actions, as $u_t$ leads to $x_{t}$, instead of leading to $x_{t+1}$ as is the notation used by \cite{lazaro-gredilla_what_2024}. Here, we use $\mathbb{H}_q[z \mid \omega] = -\int q(z, \omega) \log q(z \mid \omega) \dif z \dif \omega$ as the notation of the conditional entropy.
\begin{subequations}
    \begin{align}
        \ent{&q(x_0)} + \sum_{t=1}^T \mathbb{H}_{q}[x_t \mid x_{t-1}, u_t] \\
        &=  \ent{q(x_0)} + \sum_{t=1}^T -\int q(x_t, x_{t-1}, u_t) \log q(x_t \mid x_{t-1}, u_t) \dif x_t \dif x_{t-1} \dif u_t \\
        &=  \ent{q(x_0)} + \sum_{t=1}^T -\int q(x_t, x_{t-1}, u_t) \log \frac{q(x_t, x_{t-1}, u_t)}{q(u_t \mid x_{t-1}) q(x_{t-1})} \dif x_t \dif x_{t-1} \dif u_t \\
        &= \ent{q(x_0)} + \sum_{t=1}^T -\int q(x_t, x_{t-1}, u_t) \log \frac{q(x_t, x_{t-1}, u_t)}{q(x_{t-1})} \dif x_t \dif x_{t-1} \dif u_t  \notag \\
        &\qquad + \int q(x_t, x_{t-1}, u_t) \log q(u_t \mid x_{t-1}) \dif x_t \dif x_{t-1} \dif u_t  \\
        &= \ent{q(x_0)} + \sum_{t=1}^T \underbrace{-\int q(x_t, x_{t-1}, u_t) \log q(x_t, u_t\mid x_{t-1} ) \dif x_t \dif x_{t-1} \dif u_t}_{\mathbb{H}_{q}[x_t, u_t\mid x_{t-1} ]}  \notag \\
        &\qquad + \sum_{t=1}^T\int q(x_{t-1}, u_t) \log \frac{q(x_{t-1}, u_t)}{q(x_{t-1})}  \dif x_{t-1} \dif u_t  \\
        &= \underbrace{\ent{q(x_0)} + \sum_{t=1}^T\mathbb{H}_{q}[x_t, u_t\mid x_{t-1} ]}_{\ent{q}} + \sum_{t=1}^T \underbrace{\int q(x_{t-1}, u_t) \log q(x_{t-1}, u_t)  \dif x_{t-1} \dif u_t}_{-\ent{q(x_{t-1}, u_t)}} \notag \\
        &\qquad - \underbrace{\int q(x_{t-1}) \log q(x_{t-1})  \dif x_{t-1} }_{-\ent{q(x_{t-1})}}\\
        &= \ent{q} + \sum_{t=1}^T \ent{q(x_{t-1})} - \ent{q(x_{t-1}, u_t)} \label{eq:gredilla_entropy_final}
    \end{align}
\end{subequations}
Since in our minimization scheme this entropy is subtracted from the objective, we subtract the additional terms in \eqref{eq:gredilla_entropy_final} from the inference objective, and we obtain $\sum_{t=1}^T  \ent{q(x_{t-1}, u_t)} -\ent{q(x_{t-1})}$ as a final additional term.
\section{Consolidating State and Observation Epistemic Priors}\label{section:proof-replacment}
\begin{lemma}\label{lemma:p_tilde_x_xtheta_equivalence}
Under the assumption that our posterior distribution factorizes as in \autoref{eq:posterior-factorization}, with
\[
\tilde{p}(x_t)=\exp\!\left\{-\ent{q(y_t\mid x_t)}\right\},\qquad
\tilde{p}(y_t,x_t)=\exp D_{\mathrm{KL}}\!\big[q(\theta\mid y_t,x_t)\,\|\,q(\theta\mid x_t)\big],
\]
define \(\tilde{p}(x_t,\theta)=\exp\!\left\{-\ent{q(y_t\mid x_t,\theta)}\right\}\). Then
\begin{equation}\label{eq:equivalence-identity}
-\!\int q(x_t)\log\tilde{p}(x_t)\,\dif x_t
\;-\!\int q(y_t,x_t)\log\tilde{p}(y_t,x_t)\,\dif y_t\,\dif x_t
\;=\;
-\!\int q(x_t,\theta)\log\tilde{p}(x_t,\theta)\,\dif x_t\,\dif\theta.
\end{equation}
\end{lemma}

\begin{proof}
\begin{subequations}
\begin{align}
S_t
&:= -\!\int q(x_t)\log\tilde{p}(x_t)\,\dif x_t
    \;-\!\int q(y_t,x_t)\log\tilde{p}(y_t,x_t)\,\dif y_t\,\dif x_t \\[2pt]
&= \ent{q(y_t,x_t,\theta)} - \ent{q(x_t,\theta)}\qquad\text{(by Lemmas \ref{lemma:p_tilde_x} and \ref{lemma:p_tilde_x_y})}\\[2pt].
\end{align}
\end{subequations}
On the other hand,
\begin{subequations}
\begin{align}
R_t
&:= -\!\int q(x_t,\theta)\log\tilde{p}(x_t,\theta)\,\dif x_t\,\dif\theta
   = \int q(x_t,\theta)\,\ent{q(y_t\mid x_t,\theta)}\,\dif x_t\,\dif\theta \\[2pt]
&= -\!\int q(y_t,x_t,\theta)\,\log q(y_t\mid x_t,\theta)\,\dif y_t\,\dif x_t\,\dif\theta \\[2pt]
&= -\!\int q(y_t,x_t,\theta)\,\big(\log q(y_t,x_t,\theta)-\log q(x_t,\theta)\big)\,\dif y_t\,\dif x_t\,\dif\theta \\[2pt]
&= \ent{q(y_t,x_t,\theta)} - \ent{q(x_t,\theta)}.
\end{align}
\end{subequations}
Thus \(S_t=R_t\), which proves \eqref{eq:equivalence-identity}.
\end{proof}

\section{Proof of \autoref{thm:update-scheme}}\label{appd:proof-scheme-theorem}

In this appendix, we establish the stationary conditions for the Active Inference message-passing scheme presented in \autoref{thm:update-scheme}. The proof proceeds by deriving the first-order optimality conditions for each coordinate in the adjusted system. 

We begin by expanding the Bethe Free Energy for the generative model \eqref{eq:generative_model} and establishing the necessary consistency constraints in \eqref{eq:local-consistency}. \autoref{subsec:deg-of-marginal-system} demonstrates that the standard region coordinates lead to a degenerate optimization problem (\autoref{thm:obs-degeneracy-aug}), motivating the introduction of the channel variable $r_{y|x\theta,t}$. \autoref{subsec:conditional-adjusted-system} and \autoref{subsec:multiplier-identification} then derive the stationary conditions with respect to each coordinate—the observation factor belief $q_{y,t}$ (\autoref{lem:qy-stationarity-channel}), the channel $r_{y|x\theta,t}$ (\autoref{lem:r-stationarity-channel}), and the dynamics factor belief $q_{\mathrm{dyn},t}$ (\autoref{lem:qdyn-stationarity-min})—as well as the identifications of the Lagrange multipliers $\Lambda_{x\theta}$ (\autoref{lem:Lambda-xth-form}) and $\Lambda_{\mathrm{trip}}$ (\autoref{lem:Lambda-trip-form}).

\autoref{thm:update-scheme} follows directly from these results: equations \eqref{eq:AI-scheme-qy}--\eqref{eq:AI-scheme-Ltrip} are simply the collected stationary conditions established in \autoref{lem:qy-stationarity-channel}--\autoref{lem:Lambda-trip-form} below, expressed in the notation of the main text with the goal priors $\hat{p}_y(y_t)$ explicitly included.

BFE for the model \eqref{eq:generative_model} then reads
\begin{equation}\label{eq:bethe-our-model-updated}
\begin{aligned}
F_{p}^{\mathrm{Bethe}}[q]
&= \KL{q_\theta(\theta)}{p(\theta)} \;+\; \KL{q_{x_0}(x_0)}{p(x_0)} \\[2pt]
&\quad + \sum_{t=1}^T \Big[
\KL{q_{y,t}(y_t,x_t,\theta)}{p(y_t\mid x_t,\theta)}
+ \KL{q_{\mathrm{dyn},t}(x_t,x_{t-1},\theta,u_t)}{p(x_t\mid x_{t-1},\theta,u_t)}\\[-2pt]
&\qquad\qquad\qquad\quad
+ \KL{q_{u,t}(u_t)}{p(u_t)}
+ \KL{q_{x,\hat p,t}(x_t)}{\hat p_x(x_t)}
+ \KL{q_{y,\hat p,t}(y_t)}{\hat p_y(y_t)}
\Big] \\[4pt]
&\quad + (d_{\theta}-1)\,\ent{q_\theta(\theta)}
+ (d_{x_0}-1)\ent{q_{x_0}(x_0)} \\[2pt]
&\quad + \sum_{t=1}^T \Big[
(d_{x_t}-1)\ent{q_{x_t}(x_t)}
+ (d_{y_t}-1)\ent{q_{y_t}(y_t)}
+ (d_{u_t}-1)\ent{q_{u_t}(u_t)}
\Big].
\end{aligned}
\end{equation}

The variable-node degrees implied by \eqref{eq:generative_model} and the factorization above are
\begin{equation}\label{eq:degrees-our-model-updated}
\begin{aligned}
&d_{\theta}=1+2T \quad\text{(prior, $T$ obs, $T$ dyn)},\qquad
d_{x_0}=2 \quad\text{(prior, dyn $t{=}1$)},\\
&d_{y_t}=2 \quad\text{(obs, goal prior on $y_t$)},\qquad
d_{u_t}=2 \quad\text{(action prior, dyn $t$)},\\
&d_{x_t}=
\begin{cases}
4,& 1\le t\le T-1\quad\text{(obs $t$, dyn $t$, dyn $t{+}1$, goal prior on $x_t$)},\\
3,& t=T\quad\text{(obs $T$, dyn $T$, goal prior on $x_T$)}.
\end{cases}
\end{aligned}
\end{equation}

(For unary factors \(f_{\theta}, f_{x_0}, f_{u,t}, f_{x,\hat p,t}, f_{y,\hat p,t}\) we identify the factor belief with the adjacent singleton.)
With this notation, the Bethe Free Energy in \eqref{eq:bethe-our-model-updated} is the specialization of the general BFE in \autoref{subsec:bethe-free-energy} to \eqref{eq:generative_model}.

Local consistency requires that, for every factor \(a\) and every variable \(i\in s_a\),
\[
\int q_a(s_a)\,\mathrm{d}s_{a\setminus i}\;=\;q_i(s_i).
\]
\begin{subequations}\label{eq:local-consistency}
\begin{alignat}{2}
&\textbf{Observation }(y,t):\quad
&&\int q_{y,t}(y_t,x_t,\theta)\,\mathrm{d}x_t\,\mathrm{d}\theta = q_{y_t}(y_t),\\
&&&\int q_{y,t}(y_t,x_t,\theta)\,\mathrm{d}y_t\,\mathrm{d}\theta = q_{x_t}(x_t),\\
&&&\int q_{y,t}(y_t,x_t,\theta)\,\mathrm{d}y_t\,\mathrm{d}x_t = q_{\theta}(\theta).
\\[4pt]
&\textbf{Dynamics }(\mathrm{dyn},t):\quad
&&\int q_{\mathrm{dyn},t}(x_t,x_{t-1},\theta,u_t)\,\mathrm{d}x_{t-1}\,\mathrm{d}u_t\,\mathrm{d}\theta = q_{x_t}(x_t),\\
&&&\int q_{\mathrm{dyn},t}(x_t,x_{t-1},\theta,u_t)\,\mathrm{d}x_t\,\mathrm{d}u_t\,\mathrm{d}\theta = q_{x_{t-1}}(x_{t-1}),\\
&&&\int q_{\mathrm{dyn},t}(x_t,x_{t-1},\theta,u_t)\,\mathrm{d}x_t\,\mathrm{d}x_{t-1}\,\mathrm{d}\theta = q_{u_t}(u_t),\\
&&&\int q_{\mathrm{dyn},t}(x_t,x_{t-1},\theta,u_t)\,\mathrm{d}x_t\,\mathrm{d}x_{t-1}\,\mathrm{d}u_t = q_{\theta}(\theta).
\\[4pt]
&\textbf{Unary priors/goals}:\quad
&&q_{u,t}(u_t)=q_{u_t}(u_t),\quad
q_{x,\hat p,t}(x_t)=q_{x_t}(x_t),\quad
q_{y,\hat p,t}(y_t)=q_{y_t}(y_t).
\end{alignat}
\end{subequations}
(For unary factors we identify the factor belief with the adjacent singleton belief; the equality is enforced by the corresponding consistency constraint.)

\subsection{Degeneracy of the marginal scheme}\label{subsec:deg-of-marginal-system}

\begin{theorem}[Degeneracy persists under the augmented region coordinates]\label{thm:obs-degeneracy-aug}
Fix \(t\). Consider the Bethe-form objective specialized to \eqref{eq:generative_model}, augmented by the Active Inference entropic correction
\[
+\;\ent{q(y_t,x_t,\theta)}\;-\;\ent{q(x_t,\theta)}\;+\;\ent{q(x_{t-1},u_t)}\;-\;\ent{q(x_t,x_{t-1},u_t)}.
\]
Introduce the auxiliary region beliefs $q_{\mathrm{sep},t}(x_t,\theta)$ (defined in \autoref{eq:aux-beliefs}) with the consistency constraints
\begin{subequations}\label{eq:obs-block-constraints}
\begin{align}
&\int q_{y,t}(y_t,x_t,\theta)\,\mathrm{d}x_t\,\mathrm{d}\theta = q_{y_t}(y_t), \label{eq:obs-y-cons}\\
&\int q_{y,t}(y_t,x_t,\theta)\,\mathrm{d}y_t = q_{\mathrm{sep},t}(x_t,\theta). \label{eq:obs-sep-cons}
\end{align}
\end{subequations}
(Region beliefs \(q_{\mathrm{pair},t}\) and \(q_{\mathrm{trip},t}\) are tied to the dynamics block and do not appear in \eqref{eq:obs-block-constraints}.) Then any stationary point of the Lagrangian with respect to the observation-factor belief \(q_{y,t}(y_t,x_t,\theta)\) satisfies
\begin{equation}\label{eq:obs-stationarity-sep}
-\log p(y_t\mid x_t,\theta)\;+\;\lambda_y(y_t)\;+\;\lambda_{\mathrm{sep}}(x_t,\theta)\;=\;0,
\end{equation}
for some multipliers \(\lambda_y(\cdot)\), \(\lambda_{\mathrm{sep}}(\cdot,\cdot)\).
Consequently:
\begin{enumerate}
\item If \(p(y_t\mid x_t,\theta)\) is not separable as \(a_t(y_t)\,b_t(x_t,\theta)\), the system is infeasible (no interior solution for \(q_{y,t}\)).
\item If \(p(y_t\mid x_t,\theta)=a_t(y_t)\,b_t(x_t,\theta)\) is separable, the observation block is affine in \(q_{y,t}\) (its second variation w.r.t.\ \(q_{y,t}\) vanishes), hence stationary points are non-unique (a flat face of the feasible polytope).
\end{enumerate}
\end{theorem}

\begin{proof}
The \(q_{y,t}\)-dependent part of the objective is
\[
\int q_{y,t}\log\frac{q_{y,t}}{p(y_t\mid x_t,\theta)}\;+\;\ent{q(y_t,x_t,\theta)},
\]
since \(-\ent{q(x_t,\theta)}\), \(+\ent{q(x_{t-1},u_t)}\), and \(-\ent{q(x_t,x_{t-1},u_t)}\) do not depend on \(q_{y,t}\).
The \(\int q_{y,t}\log q_{y,t}\) from the KL term cancels exactly with \(+\ent{q(y_t,x_t,\theta)}\), leaving the linear functional \(-\int q_{y,t}\log p(y_t\mid x_t,\theta)\).
Add constraints \eqref{eq:obs-y-cons}–\eqref{eq:obs-sep-cons} with multipliers \(\lambda_y(\cdot)\), \(\lambda_{\mathrm{sep}}(\cdot,\cdot)\).
Taking the functional derivative yields \eqref{eq:obs-stationarity-sep}.
Exponentiating shows that a solution exists only if \(p(y_t\mid x_t,\theta)\propto e^{\lambda_y(y_t)}e^{\lambda_{\mathrm{sep}}(x_t,\theta)}\), i.e., it factorizes as \(a_t(y_t)b_t(x_t,\theta)\).
In the separable case, the absence of a \(\int q_{y,t}\log q_{y,t}\) term implies zero curvature in the \(q_{y,t}\)-direction and thus non-uniqueness; otherwise the system is infeasible. 
\end{proof}

\paragraph{Consequence.}
\autoref{thm:obs-degeneracy-aug} shows that even after introducing the augmented region coordinates \((q_{\mathrm{sep},t},q_{\mathrm{pair},t},q_{\mathrm{trip},t})\), the Active Inference correction leaves the observation block degenerate: feasibility requires a separable likelihood in \(y_t\) and \((x_t,\theta)\), and otherwise the block is flat.

\subsection{Stationary conditions in the condititonal adjusted system}\label{subsec:conditional-adjusted-system}

\begin{lemma}[Stationary condition for the observation factor with channel augmentation]\label{lem:qy-stationarity-channel}
Fix \(t\) and introduce the channel \(r_{y\mid x\theta,t}(y_t\mid x_t,\theta)\) with the normalization
\(\int r_{y\mid x\theta,t}(y_t\mid x_t,\theta)\,\mathrm{d}y_t=1\).
Assume the separator consistency on the overlap \((x_t,\theta)\),
\[
\int q_{y,t}(y_t,x_t,\theta)\,\mathrm{d}y_t
\;=\;
q_{x\theta}(x_t,\theta)
\;=\;
\int q_{\mathrm{dyn},t}(x_t,x_{t-1},\theta,u_t)\,\mathrm{d}x_{t-1}\mathrm{d}u_t,
\]
and the \(y_t\)-singleton consistency
\[
\int q_{y,t}(y_t,x_t,\theta)\,\mathrm{d}x_t\,\mathrm{d}\theta
\;=\;
q_{y_t}(y_t).
\]
Consider the observation-block objective (holding \(r_{y\mid x\theta,t}\) fixed):
\[
\int q_{y,t}\log\frac{q_{y,t}}{p(y_t\mid x_t,\theta)}\,
+\,
\int q_{y,t}\bigl[-\log r_{y\mid x\theta,t}(y_t\mid x_t,\theta)\bigr],
\]
plus Lagrange terms for the two consistency constraints. Then the stationarity
with respect to \(q_{y,t}(y_t,x_t,\theta)\) is
\begin{equation}\label{eq:qy-stationarity}
\log q_{y,t}(y_t,x_t,\theta)\;-\;\log p(y_t\mid x_t,\theta)\;-\;\log r_{y\mid x\theta,t}(y_t\mid x_t,\theta)
\;+\;\lambda_y(y_t)\;+\;\lambda_{x\theta}(x_t,\theta)\;=\;0,
\end{equation}
for some multipliers \(\lambda_y(\cdot)\) and \(\lambda_{x\theta}(\cdot,\cdot)\). Equivalently,
\begin{equation}\label{eq:qy-proportional}
q_{y,t}(y_t,x_t,\theta)
\;\propto\;
p(y_t\mid x_t,\theta)\;r_{y\mid x\theta,t}(y_t\mid x_t,\theta)\;
\exp\!\bigl\{-\lambda_y(y_t)\bigr\}\;
\exp\!\bigl\{-\lambda_{x\theta}(x_t,\theta)\bigr\},
\end{equation}
with \(\lambda_y,\lambda_{x\theta}\) chosen to satisfy the two projection constraints above.
\end{lemma}

\begin{proof}
Form the partial Lagrangian (omitting the redundant normalization of \(q_{y,t}\)):
\begin{align*}
\mathcal{L}[q_{y,t}]
&=\int q_{y,t}\log\frac{q_{y,t}}{p(y_t\mid x_t,\theta)}\;+\;\int q_{y,t}\bigl[-\log r_{y\mid x\theta,t}\bigr]\\
&\quad+\int \lambda_y(y_t)\!\left(\int q_{y,t}\,\mathrm{d}x_t\,\mathrm{d}\theta - q_{y_t}(y_t)\right)\mathrm{d}y_t\\
&\quad+\iint \lambda_{x\theta}(x_t,\theta)\!\left(\int q_{y,t}\,\mathrm{d}y_t - q_{x\theta}(x_t,\theta)\right)\mathrm{d}x_t\,\mathrm{d}\theta.
\end{align*}
Taking the functional derivative w.r.t.\ \(q_{y,t}\) and setting it to zero yields
\[
\log q_{y,t}-\log p(y_t\mid x_t,\theta)-\log r_{y\mid x\theta,t}
+1+\lambda_y(y_t)+\lambda_{x\theta}(x_t,\theta)=0,
\]
where the constant \(+1\) can be absorbed into either multiplier. This is \eqref{eq:qy-stationarity}. Exponentiating gives \eqref{eq:qy-proportional}, and the multipliers are determined by enforcing the two linear projection constraints. 
\end{proof}

\begin{lemma}[Stationary condition for the observation channel]\label{lem:r-stationarity-channel}
Fix \(t\). Let the channel \(r_{y\mid x\theta,t}(y_t\mid x_t,\theta)\) satisfy the row–normalization
\begin{equation}\label{eq:r-normalization}
\int r_{y\mid x\theta,t}(y_t\mid x_t,\theta)\,\mathrm{d}y_t = 1\qquad\text{for all }(x_t,\theta),
\end{equation}
and assume separator consistency on the overlap \((x_t,\theta)\):
\begin{equation}\label{eq:sep-consistency}
\int q_{y,t}(y_t,x_t,\theta)\,\mathrm{d}y_t \;=\; q_{\mathrm{sep},t}(x_t,\theta).
\end{equation}
Consider the channel objective (holding \(q_{y,t}\) and \(q_{\mathrm{sep},t}\) fixed)
\begin{equation}\label{eq:channel-objective}
\mathcal{J}[r]
\;=\;
\int q_{y,t}(y_t,x_t,\theta)\,\bigl[-\log r_{y\mid x\theta,t}(y_t\mid x_t,\theta)\bigr]\,
\mathrm{d}y_t\,\mathrm{d}x_t\,\mathrm{d}\theta,
\end{equation}
subject to \eqref{eq:r-normalization}. Then a stationary point is given pointwise by
\begin{equation}\label{eq:r-star}
r^{\star}_{y\mid x\theta,t}(y_t\mid x_t,\theta)
\;=\;
\frac{q_{y,t}(y_t,x_t,\theta)}{q_{\mathrm{sep},t}(x_t,\theta)}
\qquad\text{whenever }q_{\mathrm{sep},t}(x_t,\theta)>0,
\end{equation}
with the convention that if \(q_{\mathrm{sep},t}(x_t,\theta)=0\), the entire row
\(r_{y\mid x\theta,t}(\cdot\mid x_t,\theta)\) can be chosen arbitrarily as any
probability distribution (it does not affect \(\mathcal{J}\)).
\end{lemma}

\begin{proof}
Form the Lagrangian with a row multiplier \(\lambda(x_t,\theta)\) enforcing \eqref{eq:r-normalization}:
\[
\mathcal{L}[r]
=
\int q_{y,t}(y_t,x_t,\theta)\,\bigl[-\log r(y_t\mid x_t,\theta)\bigr]
+\iint \lambda(x_t,\theta)\!\left(\int r(y_t\mid x_t,\theta)\,\mathrm{d}y_t -1\right)\mathrm{d}x_t\,\mathrm{d}\theta.
\]
Taking the functional derivative and setting it to zero yields, pointwise in \((y_t,x_t,\theta)\),
\[
-\frac{q_{y,t}(y_t,x_t,\theta)}{r^\star(y_t\mid x_t,\theta)}+\lambda(x_t,\theta)=0
\quad\Rightarrow\quad
r^\star(y_t\mid x_t,\theta)=\frac{q_{y,t}(y_t,x_t,\theta)}{\lambda(x_t,\theta)}.
\]
Imposing the row–normalization \eqref{eq:r-normalization} and using \eqref{eq:sep-consistency},
\[
1=\int r^\star(y_t\mid x_t,\theta)\,\mathrm{d}y_t
=\frac{\int q_{y,t}(y_t,x_t,\theta)\,\mathrm{d}y_t}{\lambda(x_t,\theta)}
=\frac{q_{\mathrm{sep},t}(x_t,\theta)}{\lambda(x_t,\theta)},
\]
so \(\lambda(x_t,\theta)=q_{\mathrm{sep},t}(x_t,\theta)\), giving \eqref{eq:r-star}. If \(q_{\mathrm{sep},t}(x_t,\theta)=0\), then the row of \(q_{y,t}\) is identically zero and \(\mathcal{J}\) is unaffected by the choice of \(r(\cdot\mid x_t,\theta)\).
\end{proof}

\begin{lemma}[Stationary condition for the dynamics factor (minimal projections)]\label{lem:qdyn-stationarity-min}
Fix \(t\). Let \(q_{\mathrm{dyn},t}(x_t,x_{t-1},u_t,\theta)\) be the dynamics–factor belief and introduce the region beliefs
\(q_{x\theta,t}(x_t,\theta)\) and \(q_{\mathrm{trip},t}(x_t,x_{t-1},u_t)\) with projection constraints
\begin{subequations}\label{eq:qdyn-projections-min}
\begin{align}
&\int q_{\mathrm{dyn},t}(x_t,x_{t-1},\theta,u_t)\,\mathrm{d}x_{t-1}\,\mathrm{d}u_t
= q_{x\theta,t}(x_t,\theta), \label{eq:qdyn-to-xth-min}\\
&\int q_{\mathrm{dyn},t}(x_t,x_{t-1},\theta,u_t)\,\mathrm{d}\theta
= q_{\mathrm{trip},t}(x_t,x_{t-1},u_t). \label{eq:qdyn-to-trip-min}
\end{align}
\end{subequations}
(Separately, enforce the pair–trip relation \(\int q_{\mathrm{trip},t}(x_t,x_{t-1},u_t)\,\mathrm{d}x_t = q_{\mathrm{pair},t}(x_{t-1},u_t)\) in the \(q_{\mathrm{trip},t}\) block.)
The \(q_{\mathrm{dyn},t}\)-dependent part of the objective is the KL term
\[
\int q_{\mathrm{dyn},t}\,\log\frac{q_{\mathrm{dyn},t}}{p(x_t\mid x_{t-1},\theta,u_t)}\,\mathrm{d}x_t\,\mathrm{d}x_{t-1}\,\mathrm{d}u_t\,\mathrm{d}\theta,
\]
while the entropic corrections only involve the region beliefs.
Form the partial Lagrangian with multipliers \(\Lambda_{x\theta}(x_t,\theta)\) and \(\Lambda_{\mathrm{trip}}(x_t,x_{t-1},u_t)\) enforcing \eqref{eq:qdyn-to-xth-min}–\eqref{eq:qdyn-to-trip-min}.
Then the stationarity w.r.t.\ \(q_{\mathrm{dyn},t}\) is
\begin{equation}\label{eq:qdyn-stationarity-min-eq}
\log q_{\mathrm{dyn},t}(x_t,x_{t-1},\theta,u_t)
-\log p(x_t\mid x_{t-1},\theta,u_t)
+\Lambda_{x\theta}(x_t,\theta)
+\Lambda_{\mathrm{trip}}(x_t,x_{t-1},u_t)=0,
\end{equation}
up to an additive constant, hence
\begin{equation}\label{eq:qdyn-proportional-min}
q_{\mathrm{dyn},t}(x_t,x_{t-1},\theta,u_t)
\;\propto\;
p(x_t\mid x_{t-1},\theta,u_t)\,
\exp\!\big\{-\Lambda_{x\theta}(x_t,\theta)\big\}\,
\exp\!\big\{-\Lambda_{\mathrm{trip}}(x_t,x_{t-1},u_t)\big\}.
\end{equation}
The multipliers are determined by the two projection constraints \eqref{eq:qdyn-projections-min}, while the consistency
\(\int q_{\mathrm{trip},t}\,\mathrm{d}x_t = q_{\mathrm{pair},t}\) is enforced in the \(q_{\mathrm{trip},t}\) variation and does not appear in \eqref{eq:qdyn-proportional-min}.
\end{lemma}

\subsection{Identification of the Lagrange multipliers}\label{subsec:multiplier-identification}

\begin{lemma}[Identification of the dynamics–side separator multiplier]\label{lem:Lambda-xth-form}
Fix \(t\). Assume the channel \(r_{y\mid x\theta,t}(y_t\mid x_t,\theta)\) is row–normalized and impose separator consistency
\[
\int q_{y,t}(y_t,x_t,\theta)\,\mathrm{d}y_t \;=\; q_{\mathrm{sep},t}(x_t,\theta)
\;=\; \int q_{\mathrm{dyn},t}(x_t,x_{t-1},\theta,u_t)\,\mathrm{d}x_{t-1}\,\mathrm{d}u_t.
\]
At any stationary point of the observation block, there exists a slice–constant \(C_t>0\) such that
\begin{equation}\label{eq:Lambda-xth-statement}
\exp\!\big\{-\Lambda_{x\theta}(x_t,\theta)\big\}
\;=\;
C_t\,
\frac{\displaystyle \int p(y_t\mid x_t,\theta)\,r_{y\mid x\theta,t}(y_t\mid x_t,\theta)\,e^{-\lambda_y(y_t)}\,\mathrm{d}y_t}
{\;q_{\mathrm{sep},t}(x_t,\theta)\;},
\end{equation}
where \(\lambda_y(\cdot)\) is the Lagrange multiplier enforcing
\(\int q_{y,t}(y_t,x_t,\theta)\,\mathrm{d}x_t\,\mathrm{d}\theta = q_{y_t}(y_t)\).
Equivalently,
\[
\Lambda_{x\theta}(x_t,\theta)
\;=\;
\log q_{\mathrm{sep},t}(x_t,\theta)
\;-\;\log\!\Big(\!\int p(y_t\mid x_t,\theta)\,r_{y\mid x\theta,t}(y_t\mid x_t,\theta)\,e^{-\lambda_y(y_t)}\,\mathrm{d}y_t\Big)
\;+\;c_t,
\]
with \(c_t=-\log C_t\) independent of \((x_t,\theta)\).
\end{lemma}

\begin{proof}
From the observation–block stationarity (Lemma~\ref{lem:qy-stationarity-channel}), there exists a slice–constant
\(\kappa_t>0\) such that
\begin{equation}\label{eq:qy-sta}
q_{y,t}(y_t,x_t,\theta)
=\kappa_t\, p(y_t\mid x_t,\theta)\, r_{y\mid x\theta,t}(y_t\mid x_t,\theta)\,
e^{-\lambda_y(y_t)}\, e^{-\lambda_{x\theta}(x_t,\theta)}.
\end{equation}
Integrate \eqref{eq:qy-sta} over \(y_t\) and use the left separator equality to obtain
\[
q_{\mathrm{sep},t}(x_t,\theta)
=\kappa_t\, e^{-\lambda_{x\theta}(x_t,\theta)}
\underbrace{\int p(y_t\mid x_t,\theta)\, r_{y\mid x\theta,t}(y_t\mid x_t,\theta)\, e^{-\lambda_y(y_t)}\,\mathrm{d}y_t}_{=:~I_t(x_t,\theta)}.
\]
Solve for \(e^{-\lambda_{x\theta}(x_t,\theta)}\):
\[
e^{-\lambda_{x\theta}(x_t,\theta)}=\frac{q_{\mathrm{sep},t}(x_t,\theta)}{\kappa_t\, I_t(x_t,\theta)}.
\]
On the other hand, varying the separator \(q_{\mathrm{sep},t}\) in the Lagrangian for the two equalities gives
\(\lambda_{x\theta}(x_t,\theta)+\Lambda_{x\theta}(x_t,\theta)=0\), hence
\(e^{-\Lambda_{x\theta}(x_t,\theta)}=e^{\lambda_{x\theta}(x_t,\theta)}\).
Combining the two displays,
\[
e^{-\Lambda_{x\theta}(x_t,\theta)}
=e^{\lambda_{x\theta}(x_t,\theta)}
=\frac{\kappa_t\, I_t(x_t,\theta)}{q_{\mathrm{sep},t}(x_t,\theta)}
\;=\;
C_t\,
\frac{\displaystyle \int p(y_t\mid x_t,\theta)\,r_{y\mid x\theta,t}(y_t\mid x_t,\theta)\,e^{-\lambda_y(y_t)}\,\mathrm{d}y_t}
{\;q_{\mathrm{sep},t}(x_t,\theta)\;},
\]
with \(C_t=\kappa_t\) independent of \((x_t,\theta)\). Taking logs yields the equivalent additive form with
\(c_t=-\log C_t\).
\end{proof}

\begin{lemma}[Identification of the triplet multiplier from the dynamics-side entropic correction]\label{lem:Lambda-trip-form}
Fix \(t\). Let the region beliefs \(q_{\mathrm{trip},t}(x_t,x_{t-1},u_t)\) and \(q_{\mathrm{pair},t}(x_{t-1},u_t)\) satisfy the coupling constraint
\begin{equation}\label{eq:pair-trip-coupling}
\int q_{\mathrm{trip},t}(x_t,x_{t-1},u_t)\,\mathrm{d}x_t \;=\; q_{\mathrm{pair},t}(x_{t-1},u_t),
\end{equation}
and let the projection from the dynamics factor be enforced via the multiplier
\(\Lambda_{\mathrm{trip}}(x_t,x_{t-1},u_t)\) on
\(\int q_{\mathrm{dyn},t}\,\mathrm{d}\theta - q_{\mathrm{trip},t}=0\).
Suppose the objective contains the dynamics-side entropic correction
\(+\,H[q_{\mathrm{pair},t}] - H[q_{\mathrm{trip},t}]\).
Then, at any stationary point, there exists a slice-constant \(C_t>0\) such that
\begin{equation}\label{eq:Lambda-trip-statement}
\exp\!\big\{-\Lambda_{\mathrm{trip}}(x_t,x_{t-1},u_t)\big\}
\;=\;
C_t\;\frac{q_{\mathrm{pair},t}(x_{t-1},u_t)}{q_{\mathrm{trip},t}(x_t,x_{t-1},u_t)}\,,
\end{equation}
equivalently,
\[
\Lambda_{\mathrm{trip}}(x_t,x_{t-1},u_t)
\;=\; \log q_{\mathrm{trip},t}(x_t,x_{t-1},u_t)\;-\;\log q_{\mathrm{pair},t}(x_{t-1},u_t)\;+\;c_t,
\]
with \(c_t=-\log C_t\) independent of \((x_t,x_{t-1},u_t)\).
\end{lemma}

\begin{proof}
Consider the part of the Lagrangian involving \(q_{\mathrm{trip},t}\) and \(q_{\mathrm{pair},t}\):
\begin{align*}
\mathcal{L}
&=\underbrace{\int q_{\mathrm{trip},t}\log q_{\mathrm{trip},t}}_{\text{from }-H[q_{\mathrm{trip},t}]}
\;-\;\underbrace{\int q_{\mathrm{pair},t}\log q_{\mathrm{pair},t}}_{\text{from }+H[q_{\mathrm{pair},t}]} \\
&\quad+\;\iint\!\Xi_{\mathrm{pair}}(x_{t-1},u_t)\!\left(\int q_{\mathrm{trip},t}\,\mathrm{d}x_t - q_{\mathrm{pair},t}\right)\mathrm{d}x_{t-1}\mathrm{d}u_t \\
&\quad-\;\int \Lambda_{\mathrm{trip}}\, q_{\mathrm{trip},t}
\end{align*}
where the last term comes from the projection \(\int q_{\mathrm{dyn},t}\mathrm{d}\theta - q_{\mathrm{trip},t}=0\).
Varying w.r.t.\ \(q_{\mathrm{trip},t}\) gives
\[
\frac{\delta \mathcal{L}}{\delta q_{\mathrm{trip},t}}
:\quad \log q_{\mathrm{trip},t}+1\;+\;\Xi_{\mathrm{pair}}(x_{t-1},u_t)\;-\;\Lambda_{\mathrm{trip}}(x_t,x_{t-1},u_t)=0.
\]
Varying w.r.t.\ \(q_{\mathrm{pair},t}\) yields
\[
\frac{\delta \mathcal{L}}{\delta q_{\mathrm{pair},t}}
:\quad -\big(\log q_{\mathrm{pair},t}+1\big)\;-\;\Xi_{\mathrm{pair}}(x_{t-1},u_t)=0
\;\;\Rightarrow\;\;
\Xi_{\mathrm{pair}}(x_{t-1},u_t)= -\big(\log q_{\mathrm{pair},t}(x_{t-1},u_t)+1\big).
\]
Eliminating \(\Xi_{\mathrm{pair}}\) in the first equation gives
\[
\Lambda_{\mathrm{trip}}(x_t,x_{t-1},u_t)
= \log q_{\mathrm{trip},t}(x_t,x_{t-1},u_t)\;-\;\log q_{\mathrm{pair},t}(x_{t-1},u_t)\;+\;c_t,
\]
where the slice-constant \(c_t\) (absorbing the \(+1\) terms and any normalization) is independent of \((x_t,x_{t-1},u_t)\).
Exponentiating yields \eqref{eq:Lambda-trip-statement} with \(C_t=e^{-c_t}\).
\end{proof}
\end{document}